\documentclass[lettersize,journal]{IEEEtran}
\usepackage{amsmath,amsfonts,amsthm}
\usepackage{algorithmic}
\usepackage{algorithm}
\renewcommand{\algorithmicrequire}{ \textbf{Input:}}     

\usepackage{array}
\usepackage[caption=false,font=footnotesize,labelfont=sf,textfont=sf]{subfig}
\usepackage{textcomp}
\usepackage{stfloats}
\usepackage{url}
\usepackage{verbatim}
\usepackage{graphicx}
\usepackage{cite}
\usepackage{makecell}

\usepackage{array}
\usepackage{multirow}
\newtheorem{definition}{Definition}
\newtheorem{theorem}{Theorem}

\newtheorem{lemma}[theorem]{Lemma}
\begin{document}

\title{HoGS: Homophily-Oriented Graph Synthesis for Local Differentially Private GNN Training}

\author{Wen Xu,~
        Zhetao Li,~\IEEEmembership{Member,~IEEE},~
        Yong Xiao,~
        Pengpeng Qiao,~
        Mianxiong Dong,~\IEEEmembership{Senior Member,~IEEE},~
        Kaoru Ota,~\IEEEmembership{Member,~IEEE}
\thanks{This work was supported in part by the National Natural Science Foundation of China under Grant W2411053 and Grant U23B2027; in part by JSPS KAKENHI Grant Numbers JP22K11989, JP24K14910, JP25K00139 and JST ASPIRE Grant Number JPMJAP2344. (Corresponding author: Zhetao Li.)
}
\thanks{Wen Xu is with the College of Information Science and Technology, Jinan University, Guangzhou 510632, China, and also with the Department of Sciences and Informatics, Muroran Institute of Technology, Muroran 050-8585, Japan \protect (E-mail: wenxu018@gmail.com).}
\thanks{Zhetao Li and Yong Xiao are with the College of Information Science and Technology, Jinan University, Guangzhou 510632, China \protect (E-mail: liztchina@hotmail.com; xiaoyong@stu2022.jnu.edu.cn).}
\thanks{Pengpeng Qiao is with the School of Computing, Institute of Science Tokyo, Tokyo 1528550, Japan \protect (E-mail: peng2qiao@gmail.com).}
\thanks{Mianxiong Dong and Kaoru Ota are with the Department of Sciences and Informatics, Muroran Institute of Technology, Muroran 050-8585, Japan \protect (E-mail: mxdong@muroran-it.ac.jp; ota@muroran-it.ac.jp).}
}

\markboth{}%
{Shell \MakeLowercase{\textit{et al.}}: HoGS: Homophily-Oriented Graph Synthesis for Local Differentially Private GNN Training}


\maketitle

\begin{abstract}
Graph neural networks (GNNs) have demonstrated remarkable performance in various graph-based machine learning tasks by effectively modeling high-order interactions between nodes. However, training GNNs without protection may leak sensitive personal information in graph data, including links and node features. Local differential privacy (LDP) is an advanced technique for protecting data privacy in decentralized networks. Unfortunately, existing local differentially private GNNs either only preserve link privacy or suffer significant utility loss in the process of preserving link and node feature privacy. In this paper, we propose an effective LDP framework, called HoGS, which trains GNNs with link and feature protection by generating a synthetic graph. Concretely, HoGS first collects the link and feature information of the graph under LDP, and then utilizes the phenomenon of homophily in graph data to reconstruct the graph structure and node features separately, thereby effectively mitigating the negative impact of LDP on the downstream GNN training. We theoretically analyze the privacy guarantee of HoGS and conduct experiments using the generated synthetic graph as input to various state-of-the-art GNN architectures. Experimental results on three real-world datasets show that HoGS significantly outperforms baseline methods in the accuracy of training GNNs.
\end{abstract}


\section{Introduction} \label{INTRODUCTION}
\IEEEPARstart{O}{ver} the past few years, the advances of deep learning on graphs have greatly enhanced the ability to handle graph-based machine learning tasks. Among various graph learning algorithms, Graph Neural Networks (GNNs) have become a mainstream paradigm due to their inherent efficiency and strong generalization capabilities, achieving state-of-the-art performance in a wide range of fields such as recommendation systems \cite{ying2018graph,zhang2025embedding}, traffic forecasting \cite{jiang2022graph}, and signal processing \cite{he2025icgnn}. However, the graph data used for training is typically decentralized across local user devices and may contain sensitive individual information, including links and node features. For instance, in an infectious disease surveillance system, training a GNN to identify high-risk individuals requires access to private health information (e.g., pneumonia or influenza) and sensitive contact records captured by on-device location sensors. Similarly, in financial risk management, while GNNs can effectively detect money-laundering syndicates by leveraging on-device transaction logs and account attributes, such data inherently reveal confidential financial relationships and individual asset statuses stored locally. Fig. \ref{architecture} illustrates the architecture for sensitive data collection and training in decentralized environment. Thus, it is necessary to develop solutions that enable GNNs to be trained effectively for graph learning tasks without compromising the privacy of links and node features in the graph.

\begin{figure}[!t]
	\centering
	\includegraphics[width=0.85\linewidth]{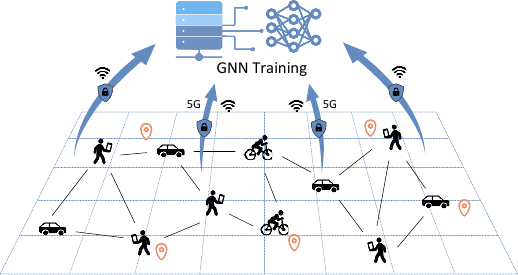}
	\caption{The architecture for sensitive data collection and training in decentralized environment.}
	\label{architecture}
\end{figure}

Local differential privacy (LDP) \cite{duchi2013local,li2024user} is a promising technique for protecting the privacy of sensitive data in decentralized networks, where each data owner shares a locally perturbed version of their data with a (possibly malicious) curator. Since LDP can ensure the local privacy of massive end entities, it has been adopted by many well-known companies, e.g., Google \cite{erlingsson2014rappor}, Apple \cite{DPTeamApple}, and Uber \cite{johnson2018towards}. However, current studies on local differential privacy GNNs mainly focus on preserving link privacy, while assuming that the curator can fully access node features \cite{hidano2022degree,zhu2023blink}. For the problem of simultaneously preserving link and node feature privacy, the few existing preliminary explorations (such as \cite{lin2022towards}) yielded poor performance by merely maintaining graph sparsity. Therefore, there is an urgent need for a novel scheme to achieve a better privacy-utility tradeoff for dual privacy scenario.

In this paper, we investigate this problem by considering a rigorous setting where both node features and graph topology (i.e., links) are treated as private information.
We focus on generating a synthetic graph that is semantically similar to the original graph under LDP, which can be used to train an effective GNN. As shown in Fig. \ref{example}, the curator reconstructs the graph based on the obfuscated local neighborhood and node features transmitted by decentralized nodes. To this end, we face the following challenges:

\begin{figure}[!t]
	\centering
	\includegraphics[width=0.7\linewidth]{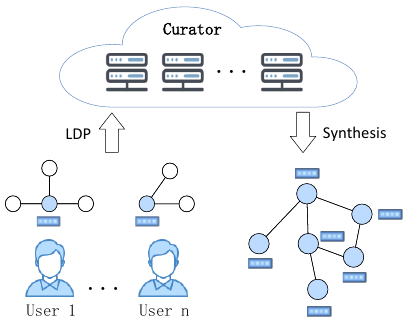}
	\caption{The example of reconstructing the graph under LDP guarantee.}
	\label{example}
\end{figure}

\begin{itemize}
    \item \textit{How to mitigate the distortion of noisy graph structure?} In the LDP setting, the curator cannot access the original graph structure, which have to be perturbed locally by the nodes. However, LDP mechanisms (e.g., randomized response \cite{warner1965randomized}) introduce significant noise that can severely distort the graph structure. In addition, a portion of the privacy budget needs to be allocated to node features, which further exacerbates the destruction of the graph structure. Since GNN training is very sensitive to changes in the links, it is not trivial to design a scheme that improve the quality of noisy graph structure.
    \item \textit{How to effectively denoise the perturbed node features?} GNNs learn a node's representation by aggregating information from its neighbors. Adding noise to node features may lead to inaccurate node representations, thereby compromising the effectiveness of GNN training. Existing methods denoise node features by calculating the mean of multi-hop neighbor features. However, this denoising can be counterproductive when the graph structure itself is noisy. Hence, it is challenging to perform effective denoising on node features while preserving graph structure privacy.
\end{itemize}

To address the above challenges, we propose HoGS (\textbf{H}omophily-\textbf{o}riented \textbf{G}raph \textbf{S}ynthesis), a novel LDP-based graph synthesis framework for training GNNs with link and feature protection. 
The core idea of the framework is to leverage homophily \cite{mcpherson2001birds} -- a common property of many real-world networks like citation and social networks -- to denoise graph topology and node features.
Specifically, we first collect perturbed adjacency lists and feature vectors from decentralized nodes to ensure LDP for links and features. Based on the phenomenon of homophily that nodes with similar features are more likely to be connected, we design a graph reconstruction mechanism that utilizes node features. This mechanism combines the cosine similarity of noisy feature vectors and link information provided by the noisy adjacency lists to determine the existence probability of each potential link, and reconstructs the graph topology by retaining links with higher probabilities. Conversely, homophily also implies that neighboring nodes often have similar features. Therefore, links can, in turn, be utilized to reconstruct node features. To mitigate the impact of link noise on feature reconstruction, we design a weighted feature aggregation mechanism that can suppress interference from unreliable connections and effectively denoise node features. In summary, our main contributions are as follows:
\begin{itemize}
    \item We present a novel and practical LDP framework HoGS to generate a synthetic graph that can be used for GNN training, while preserving the privacy of links and node features.
    \item We innovatively leverage homophily to denoise graph structure and node features in a bidirectional manner, and theoretically prove that HoGS satisfies LDP.
    \item We conduct comprehensive experiments on three GNN architectures using three real-world datasets. Experimental results demonstrate that the synthetic graph generated by HoGS can train effective GNNs.
\end{itemize}

The remainder of this paper is organized as follows. Section \ref{Preliminaries} introduces the background of GNNs and LDP and formally defines the problem. Section \ref{Our Approach} provides details of the proposed framework HoGS. In Section \ref{Algorithm Analysis}, we analyze the privacy and computational complexity of HoGS. Section \ref{Evaluation} presents and discusses extensive experimental results. In Section \ref{Related Work}, we review related work. In Section \ref{Conclusion}, we conclude the paper.

\section{Preliminaries} \label{Preliminaries}
\subsection{Graph Neural Networks}
Consider a undirected graph $G=(V,E,X)$, where $V=\{v_1,v_2,\cdots,v_n\}$ is the set of all $n$ nodes (i.e., users), $E \subseteq V \times V$ is the set of edges (i.e., links) and $X\in R^{n \times d}$ is the feature matrix. The $i$-th row of $X$ is the $d$-dimensional feature vector of node $v_i$ (denoted as $X_i$). The topology of graph $G$ can be represented as an adjacency matrix $A$, where $A_{ij}=1$ if $(v_i, v_j)\in E$, otherwise $A_{ij}=0$. In this paper, we focus on the semi-supervised node classification task on graph $G$. In this setting, $V=V_l\cup V_u$, where $V_l$ and $V_u$ are the sets of labeled and unlabeled nodes, respectively. For each $v_i \in V_l$, $Y_i$ denotes its $c$-dimensional label vector (one-hot encoding \cite{rodriguez2018beyond}), where $c$ is the number of classes. The label vectors of all $l$ nodes in $V_l$ form the label matrix $Y\in \{0,1\}^{l \times c}$.

Graph Neural Networks (GNNs) can process graph data and perform the node classification tasks. Specifically, a GNN model learns a high-dimensional representation for each node by aggregating the embedding vectors of adjacent nodes followed by a parameterized nonlinear transformation at each layer. At layer $k$, the embedding $\mathbf{h}_i^{(k)}$ of node $v_i\in V$ learned by the GNN can be represented as:
\begin{equation}
	\mathbf{h}_i^{(k)} = \Phi^{(k)} \left(\operatorname{Agg}\left(\left\{\mathbf{h}_j^{(k-1)}: \forall v_j \in \mathcal{N}(v_i)\right\}\right)\right),
\end{equation}
where $\mathcal{N}(v_i)$ denotes the set of adjacent nodes of $v_i$. $\operatorname{Agg}(\cdot)$ is a permutation invariant aggregator function, e.g., max or mean. $\Phi^{(k)}$ is a trainable transformation such as a neural network. Note that the initial embedding $\mathbf{h}_i^{(0)}=X_i$, which is the feature vector of $v_i$. The last layer outputs a $c$-dimensional embedding vector, which is then converted into the predicted probability for each label through a softmax layer.

\subsection{Local Differential Privacy}
Local Differential Privacy (LDP) \cite{duchi2013local} is an advanced privacy protection standard that aims to prevent the leakage of individual sensitive information during the data collection process. Specifically, under the LDP model, users first perturb their personal data locally through a differential privacy mechanism before sending it to an curator, thus eliminating the need for a trusted third party. Formally, LDP can be defined as follows:

\begin{definition}[$\varepsilon$-LDP]
Given $\varepsilon>0$, an algorithm $\mathcal{M}$ provides $\varepsilon$-LDP, if for any possible pair of input data $x$ and $x^{\prime}$, and for any possible output $x^{*}$ of $\mathcal{M}$, we have
\begin{equation}
    \label{LDP}
		\operatorname{Pr}\left[\mathcal{M}\left(x\right) = x^{*}\right] \leq e^{\varepsilon} \operatorname{Pr}\left[\mathcal{M}\left(x^{\prime}\right) = x^{*}\right].
\end{equation}
\end{definition}

Essentially, $\varepsilon$-LDP guarantees that an attacker cannot distinguish whether the input data is $x$ or $x^{\prime}$ with a certain confidence by observing the output of $\mathcal{M}$. The strength of privacy protection is controlled by the parameter $\varepsilon$ (called privacy budget), i.e., a smaller $\varepsilon$ indicates a stronger privacy guarantee.

\noindent
\textbf{1-Bit Mechanism.} The 1-Bit mechanism \cite{ding2017collecting} can be used for queries with binary answers. Specifically, given a private input value $x\in [\alpha,\beta]$, this mechanism outputs a noisy value $y\in \{0,1\}$ drawn from the following distribution:
\begin{equation}
    y=\left\{\begin{array}{rr}
		1, & \text { with probability } \frac{1}{e^\varepsilon + 1} + \frac{x}{\beta-\alpha}\cdot \frac{e^\varepsilon - 1}{e^\varepsilon + 1}, \\
		0, & \text { otherwise} .
	\end{array}\right.
\end{equation}

Ding et al. proved that this mechanism satisfies $\varepsilon$-LDP \cite{ding2017collecting}. In fact, when $x$ is also a binary value, the 1-Bit mechanism degenerates into Warner’s RR (Randomized Response) \cite{warner1965randomized}.

\noindent
\textbf{Composition Properties.} LDP algorithms satisfy the following composition properties, which provide privacy guarantees for multiple computations.
\begin{theorem}[sequential composition \cite{mcsherry2009privacy}]
\label{sequential composition}
    Combining multiple sub-mechanisms that satisfy LDP for $\{\varepsilon_1,\cdots, \varepsilon_k\}$ results in a mechanism satisfying $\varepsilon$-LDP, where $\varepsilon = \sum_{i=1}^{k} \varepsilon_i$.
\end{theorem}

\begin{theorem}[post-processing \cite{dwork2006calibrating}]
\label{post-processing}
    Let an algorithm $\mathcal{M}$ satisfies $\varepsilon$-LDP, then $g(\mathcal{M}(x))$ for any function $g$ still satisfies $\varepsilon$-LDP.
\end{theorem}

\subsection{Formal Problem Definition}
In our study, we are interested in the problem of learning a GNN model with both links and features privacy. In particular, we consider a network system over decentralized nodes, where each node $v_i$ locally holds its adjacency list $A_i$ (i.e., the $i$-th row of the adjacency matrix $A$) and feature vector $X_i$. Assume that an untrusted curator has access to nodes $V$ and labels $Y$, but cannot observe adjacency matrix $A$ (i.e., link information) and features $X$ in the graph, which are private information of the nodes. Thus, our goal is to design an $\varepsilon$-LDP algorithm $\mathcal{M}$ to help the curator privately collect link and feature information from the nodes, and then construct a synthetic graph $\hat{G}$ that can be used to train a GNN model for classifying unlabeled nodes in $V_u$. Table \ref{notations} summarizes the important notations used in this paper.

\begin{table}[!t]
	\renewcommand{\arraystretch}{1.2}
	\caption{Summary of Important Notations}
	\label{notations}
	\centering
	\begin{tabular}{c c}
		\hline
		Symbol & Description \\
		\hline
		$n$ & Number of nodes  \\
		$\varepsilon$ & Privacy budget \\
		$G=(V,E,X)$ & Graph with nodes $V$, edges $E$ and feature matrix $X$ \\
        $v_i$ & $i$-th node in $V$ \\
        $X_i$ & $i$-th row of $X$ (i.e., feature vector of $v_i$)  \\
        $d$ & Dimension of $X_i$  \\
		$A$ & Adjacency matrix of $G$  \\
		$\delta$ & privacy parameter \\
		$\tau$ & threshold \\
        $V_i$ & potential neighbor set of $v_i$ \\
		\hline
	\end{tabular}
\end{table}

To address privacy concerns on links and features, we adopt the concept of edge LDP \cite{qin2017generating} and introduce a formal definition of LDP for node features.

\begin{definition}[$\varepsilon$-edge LDP]
    Given $\varepsilon>0$, an algorithm $\mathcal{M}$ provides $\varepsilon$-edge LDP, if for any two possible adjacency list $A_i$ and $A_i^{\prime}$ ($i\in [n]$) that differ in one bit, and for any possible output $O$ of $\mathcal{M}$,
    \begin{equation}
    \label{edgeLDP}
		\operatorname{Pr}\left[\mathcal{M}\left(A_i\right) = O\right] \leq e^{\varepsilon} \operatorname{Pr}\left[\mathcal{M}\left(A_i^{\prime}\right) = O\right].
\end{equation}
\end{definition}

\begin{definition}[$\varepsilon$-feature LDP]
    Given $\varepsilon>0$, an algorithm $\mathcal{M}$ provides $\varepsilon$-feature LDP, if for any two possible feature vector $X_i$ and $X_i^{\prime}$ ($i\in [n]$) that differ in one bit, and for any possible output $O$ of $\mathcal{M}$,
    \begin{equation}
    \label{featureLDP}
		\operatorname{Pr}\left[\mathcal{M}\left(X_i\right) = O\right] \leq e^{\varepsilon} \operatorname{Pr}\left[\mathcal{M}\left(X_i^{\prime}\right) = O\right].
\end{equation}
\end{definition}

Edge LDP (resp. Feature LDP) protects any single bit in an adjacency list (resp. a feature vector) with privacy budget $\varepsilon$. In other words, if an algorithm provides both edge LDP and feature LDP, then any links and any features of each node has a bounded impact on the final output, thus preserving the privacy of both links and features.

\noindent
\textbf{Discussion.} Different from the user-level neighbor definition for node features \cite{sajadmanesh2021locally,lin2022towards}, feature-LDP considers any two feature vectors as neighbors if they differ by exactly one bit. Nevertheless, feature-LDP can still achieve strong indistinguishability for any features of each node in the graph, which can support the privacy requirements on node features while retaining high utility.

\section{Our Approach: HoGS} \label{Our Approach}

\subsection{Overview} \label{Overview}
To achieve GNNs learning with link and feature protection under LDP, we propose a new framework HoGS. This framework generates a high-quality synthetic graph that enables effective training of GNN models while satisfying LDP. HoGS is based on the following idea: \textit{a network system composed of graph topology and node feature information generally has the phenomenon of homophily, meaning that neighboring nodes tend to have similar features, and nodes with similar features have a higher probability of being connected \cite{mcpherson2001birds}.} Therefore, after independently injecting noise into the links and node features, we can exploit these two types of information to perform bidirectional denoising. 

Specifically, as shown in Fig. \ref{overview}, the workflow of HoGS consists of the following three phases:
\begin{itemize}
	\item \textbf{Information Collection.} In this phase, the curator does not have any information about the graph topology and node features for GNN training. Therefore, each node perturbs its adjacency list and feature vector locally and then sends this noisy information to the curator. The details of this phase are referred to Section \ref{Information Collection}.
    \item \textbf{Topology Reconstruction.} At this stage, the curator first calculates the feature similarity between nodes based on the collected feature vectors, which is regarded as the prior knowledge of node connections. Then, combined with the noisy adjacency lists, the curator uses a Bayesian estimation mechanism similar to Blink \cite{zhu2023blink} to determine the final connection probability between nodes, and then reconstruct the graph topology based on this probability. The details of this stage are described in Section \ref{Topology Reconstruction}.
    \item \textbf{Feature Reconstruction.} In this step, for each node, the curator determines its potential neighbors based on the final connection probabilities obtained in the previous phase, and uses this probability as weights to aggregate the noisy feature vectors of these potential neighbors. The aggregated results are used to reconstruct the feature vector for each node. The details of this step are in Section \ref{Feature Reconstruction}.
\end{itemize}

\begin{figure*}[!t]
	\centering
	\includegraphics[width=0.85\linewidth]{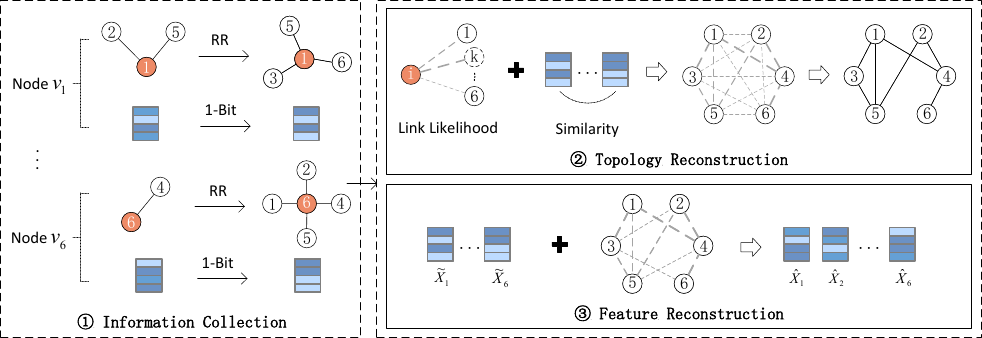}
	\caption{The overview of our HoGS framework.}
	\label{overview}
\end{figure*}

\subsection{Information Collection} \label{Information Collection}
Since the two basic information required for GNNs training: graph topology and node features, are stored locally by decentralized nodes, the curator first needs to collect them from all nodes. That is, we consider privately collecting the adjacency list and feature vector of each node, which contain detailed information about the graph topology and node features.

Based on the Theorem \ref{sequential composition}, to achieve LDP for links and features under a total privacy budget $\varepsilon$, we need to assign $\varepsilon$ to the adjacency list and feature vector. To this end, for each node, we perturb its adjacency list and feature vector with privacy budgets $\varepsilon_a$ and $\varepsilon_f$, respectively, which are controlled by privacy parameter $\delta\in [0,1]$ as $\varepsilon_a=(1-\delta)\varepsilon$ and $\varepsilon_f=\delta\varepsilon$.
Specifically, given adjacency list $A_i\in \{0,1\}^n$ of node $v_i$, we adopt the RR (Randomized Response) to flip each bit in $A_i$ with probability $p_1=\frac{1}{e^{\varepsilon_a} + 1}$, obtaining a noisy adjacency list $\tilde{A}_i\in \{0,1\}^n$. Given feature vector $X_i \in [\alpha, \beta]^d$ of node $v_i$, we adopt the 1-Bit mechanism to perturb each bit in $X_i$ with probability $p_2 = \frac{1}{e^{\varepsilon_f} + 1} + \frac{x}{\beta-\alpha}\cdot \frac{e^{\varepsilon_f} - 1}{e^{\varepsilon_f} + 1}$, resulting in a noisy feature vector $\tilde{X}_i\in \{0,1\}^d$. The node $v_i$ then sends $\tilde{A}_i$ and $\tilde{X}_i$ to the curator. This process is described in Algorithm \ref{alg:IC}.

\begin{algorithm}[!t]
	\caption{Information Collection}\label{alg:IC}
	\begin{algorithmic}[1]
		\renewcommand{\algorithmicrequire}{ \textbf{Input:}}
		\REQUIRE Adjacency list $A_i\in \{0,1\}^n$ and feature vector $X_i \in [\alpha, \beta]^d$ of node $v_i\in V$, privacy budget $\varepsilon=\varepsilon_a+\varepsilon_f$
		\ENSURE Noisy adjacency list $\tilde{A}_i\in \{0,1\}^n$ and noisy feature vector $\tilde{X}_i\in \{0,1\}^d$ of node $v_i\in V$
        \STATE \textsl{// adjacency list perturbation}
        \FOR {$j\in \{1,2,\cdots,n\}$ }
        \STATE $\tilde{A}_{ij}=\text{RR} (A_{ij}, \varepsilon_a)$;
        \ENDFOR 
        \STATE \textsl{// feature vector perturbation}
        \FOR {$k\in \{1,2,\cdots,d\}$}
        \STATE $\tilde{X}_{ik}=\text{1-Bit} (X_{ik}, \varepsilon_f)$;
        \ENDFOR 
		\RETURN $\tilde{A}_i, \tilde{X}_i$
	\end{algorithmic}
\end{algorithm}

\subsection{Topology Reconstruction} \label{Topology Reconstruction}
Now, the curator reconstructs the graph topology based on the noisy information collected from all nodes in the previous phase. A simple solution is that the curator assembles the noisy adjacency lists $\tilde{A}_1,\tilde{A}_2,\cdots,\tilde{A}_n$ into an adjacency matrix $\tilde{A}$, which can directly form the synthetic graph topology. However, as mentioned in previous studies \cite{lin2022towards,wu2022linkteller}, this will make the synthetic graph too dense to be useful. Blink \cite{zhu2023blink} proposed to use node degree information to denoise the graph structure based on the adjacency matrix. Unfortunately, to satisfy edge LDP, node degree also needs to consume privacy budget, which further increases the noise level and destroys data utility in the case where both node features and graph topology are private information.

To effectively denoise the graph topology without further splitting the privacy budget, we design a new graph reconstruction algorithm that utilizes node features as auxiliary information. The intuition behind this algorithm is that similar nodes tend to connect to each other, a property known as homophily. As mentioned before, this property has been observed in many network systems. In particular, the curator regards the similarity between noisy feature vectors as the prior link probability, which is combined with the likelihood evidence provided by the noisy adjacency matrix to obtain the posterior link probability. As shown in Algorithm \ref{alg:TR}, for any two nodes $v_i,v_j\in V$, the curator first calculates the cosine similarity $s_{ij}$ between their noisy feature vectors $\tilde{X}_i$ and $\tilde{X}_j$:
\begin{equation}
    s_{ij} = \frac{\tilde{X}_i \cdot \tilde{X}_j}{\|\tilde{X}_i\|\|\tilde{X}_j\|},
\end{equation}
where $\|\cdot\|$ represents the Euclidean norm. Note that since each bit in the noisy feature vectors is binary, $s_{ij}$ ranges from 0 to 1, i.e., $s_{ij}\in [0,1]$, which can be used as the prior probability that a link exists between $v_i$ and $v_j$.

Then, inspired by Blink \cite{zhu2023blink}, the curator evaluates the link likelihood based on the received noisy adjacency matrix and calculates the posterior probability using Bayesian estimation. Specifically, for each potential link $(v_i,v_j)$, the curator considers the corresponding two bits $\tilde{A}_{ij}$ and $\tilde{A}_{ji}$ in the noisy adjacency matrix $\tilde{A}$, and computes the likelihood of receiving the bits $(\tilde{A}_{ij},\tilde{A}_{ji})$ conditioned on whether the link $(v_i,v_j)$ exists in the true graph:
\begin{equation}
\begin{aligned}
    l_{ij} & = \operatorname{Pr}[(\tilde{A}_{ij},\tilde{A}_{ji})|(A_{ij},A_{ji}) = (1, 1)] \\
    & = \left\{\begin{array}{ll}
    p_{1}^2, & \text{if}~(\tilde{A}_{i j}, \tilde{A}_{j i})=(0,0), \\
    p_{1}(1-p_{1}), & \text{if}~(\tilde{A}_{i j}, \tilde{A}_{j i})=(0,1)~\text{or}~(1,0),\\
    (1-p_{1})^2, & \text{if}~(\tilde{A}_{i j}, \tilde{A}_{j i})=(1,1);
    \end{array} \right.
\end{aligned} 
\end{equation}
and 
\begin{equation}
\begin{aligned}
    l_{ij}^{\prime} & = \operatorname{Pr}[(\tilde{A}_{ij},\tilde{A}_{ji})|(A_{ij},A_{ji}) = (0, 0)] \\
    & = \left\{\begin{array}{ll}
	(1-p_{1})^2, & \text{if}~(\tilde{A}_{i j}, \tilde{A}_{j i})=(0,0), \\
    p_{1}(1-p_{1}), & \text{if}~(\tilde{A}_{i j}, \tilde{A}_{j i})=(0,1)~\text{or}~(1,0),\\
    p_{1}^2, & \text{if}~(\tilde{A}_{i j}, \tilde{A}_{j i})=(1,1).
	\end{array}\right.
\end{aligned} 
\end{equation}
According to the cosine similarity $s_{ij}$, and the likelihoods $l_{ij}$ and $l_{ij}^{\prime}$, the curator obtains the posterior probability of the existence of the potential link $(v_i,v_j)$ as follows:
\begin{equation}
\begin{aligned}
    P_{ij} & = \operatorname{Pr}[(A_{ij},A_{ji}) = (1, 1)|(\tilde{A}_{ij},\tilde{A}_{ji})]\\
    & = \frac{l_{ij}\cdot s_{ij}}{l_{ij}\cdot s_{ij} + l_{ij}^{\prime}\cdot (1-s_{ij})}.
\end{aligned} 
\end{equation}

Finally, to reconstruct the graph topology, the curator retains the potential link $(v_i,v_j)$ with $P_{ij}$ greater than a threshold $\tau$, i.e., $(v_i,v_j)\in \hat{E}$ if $P_{ij}\geq \tau$, where $\hat{E}$ is the set of edges in the synthetic graph $\hat{G}$.

\begin{algorithm}[!t]
	\caption{Topology Reconstruction}\label{alg:TR}
	\begin{algorithmic}[1]
		\renewcommand{\algorithmicrequire}{ \textbf{Input:}}
		\REQUIRE Noisy adjacency list $\tilde{A}_1,\tilde{A}_2,\cdots,\tilde{A}_n\in \{0,1\}^n$, noisy feature vectors $\tilde{X}_1,\tilde{X}_2,\cdots,\tilde{X}_d\in \{0,1\}^d$, threshold $\tau$
		\ENSURE Edge set $\hat{E}$ of synthetic graph $\hat{G}$
        \STATE Initialize $\hat{E}=\emptyset$;
        \FOR {$\{i,j\}\in \{1,2,\cdots,n\}^2$ }
        \STATE \textsl{// Calculate feature similarity}
        \STATE $s_{ij} = \dfrac{\tilde{X}_i \cdot \tilde{X}_j}{\|\tilde{X}_i\|\|\tilde{X}_j\|}$;
        \STATE \textsl{// Infer link probability}
        \STATE $l_{ij} = \operatorname{Pr}[(\tilde{A}_{ij},\tilde{A}_{ji})|(A_{ij},A_{ji}) = (1, 1)]$;
        \STATE $l_{ij}^{\prime} = \operatorname{Pr}[(\tilde{A}_{ij},\tilde{A}_{ji})|(A_{ij},A_{ji}) = (0, 0)]$;
        \STATE $P_{ij}=\dfrac{l_{ij}\cdot s_{ij}}{l_{ij}\cdot s_{ij} + l_{ij}^{\prime}\cdot (1-s_{ij})}$;
        \ENDFOR
        \STATE \textsl{// Generate links} 
        \STATE Add $(v_i,v_j)$ to $ \hat{E}$ if $P_{ij}\geq \tau$;
		\RETURN $\hat{E}$
	\end{algorithmic}
\end{algorithm}

\subsection{Feature Reconstruction} \label{Feature Reconstruction}
Considering that a node in the network also usually has similar features with its neighbors, the curator can further exploit this Homophily to reconstruct the feature vector of each node. Existing schemes \cite{sajadmanesh2021locally,lin2022towards} denoise the features of a node by aggregating the feature vectors of its multi-hop neighbors. However, some edges in the synthetic graph topology do not exist in the original graph. In this case, the features of a node's neighbors may differ significantly from its own, which will cause the aggregated feature vector to deviate from the true feature vector and ultimately impair the performance of the learned GNN.

Based on the above analysis, we propose a weighted feature aggregation method that aggregates the features of each node's potential neighbors using the posterior probability $P_{ij}$ as link weight to reconstruct feature vectors. To reduce noise and computational cost, we restrict the aggregation to potential links with $P_{ij}\geq 0.5$. Specifically, for each node $v_i$, we define its potential neighbor set as $V_i=\{v_j|P_{ij}\geq 0.5\}$. After that, the curator calculates the reconstructed feature vector $\hat{X}_i$ of node $v_i$ as the weighted average of the feature vectors of all nodes in $V_i$:
\begin{equation}
\label{weighted average}
\hat{X}_i=\frac{\sum_{v_j \in V_i} P_{i j} \cdot \tilde{X}_j}{\sum_{v_j \in V_i} P_{i j}}.
\end{equation}

Eq. \ref{weighted average} means that the contribution of a node's potential neighbors to its feature reconstruction is proportional to the link weight $P_{ij}$, which naturally grants greater influence to high-probability links while suppressing the noise that low-probability connections may introduce during feature reconstruction. Algorithm \ref{alg:FR} shows the feature reconstruction procedure. Note that our feature reconstruction can be performed $l$ times, which is data-driven. At this point, the synthetic graph $\hat{G}=(V,\hat{E},\hat{X})$ is generated and can be applied to the training of various GNN models. 

\begin{algorithm}[!t]
	\caption{Feature Reconstruction}\label{alg:FR}
	\begin{algorithmic}[1]
		\renewcommand{\algorithmicrequire}{ \textbf{Input:}}
		\REQUIRE Link probability matrix $P_{ij}$, noisy feature vectors $\tilde{X}_1,\tilde{X}_2,\cdots,\tilde{X}_d\in \{0,1\}^d$
		\ENSURE Feature matrix $\hat{X}$ of synthetic graph $\hat{G}$
        \STATE Initialize $\hat{X}=\emptyset$;
        \STATE Define $V_i=\{v_j|P_{ij}\geq 0.5\}$ as the potential neighbor set of node $v_i$;
        \STATE \textsl{// Generate features}
        \FOR {$i\in \{1,2,\cdots,n\}$ }
        \STATE $\hat{X}_i=\dfrac{\sum_{v_j \in V_i} P_{i j} \cdot \tilde{X}_j}{\sum_{v_j \in V_i} P_{i j}}$;
        \ENDFOR
		\RETURN $\hat{X}$
	\end{algorithmic}
\end{algorithm}

\section{Algorithm Analysis} \label{Algorithm Analysis}
\subsection{Privacy Analysis} \label{Privacy Analysis}
Here we demonstrate that HoGS satisfies $\varepsilon$-LDP. To this end, we first establish the following lemma.

\begin{lemma}
	\label{lemma}
Information Collection satisfies $\varepsilon_a$-edge LDP and $\varepsilon_f$-feature LDP.
\end{lemma}
\begin{proof}
Assume $\mathbf{a}=(a_1,\cdots,a_n)$ and $\mathbf{a}^{\prime}=(a_1^{\prime},\cdots,a_n^{\prime})$ are two adjacency lists that differ only at the $i$-th bit; $\mathbf{x}=(x_1,\cdots,x_d)$ and $\mathbf{x}^{\prime}=(x_1^{\prime},\cdots,x_d^{\prime})$ are two feature vector that differ only at the $j$-th bit. For ease of presentation, we denote the Randomized Response and the 1-Bit mechanism used in this phase as $\mathcal{M}_1$ and $\mathcal{M}_2$, respectively. Let $o=(o_1,\cdots,o_n)$ be the output of $\mathcal{M}_1$ applied to $\mathbf{a}$ and $\mathbf{a}^{\prime}$, and $s=(s_1,\cdots,s_d)$ be the output of $\mathcal{M}_2$ applied to $\mathbf{x}$ and $\mathbf{x}^{\prime}$. Then, we have:
\begin{equation}
    \frac{\operatorname{Pr}\left[\mathcal{M}_1\left(\mathbf{a}\right) = o\right]}{\operatorname{Pr}\left[\mathcal{M}_1\left(\mathbf{a}^{\prime}\right) = o\right]}=\frac{\operatorname{Pr}\left[a_i=o_i\right]}{\operatorname{Pr}\left[a_i^{\prime}=o_i\right]}\leq \frac{e^{\varepsilon_a}/(e^{\varepsilon_a}+1)}{1/(e^{\varepsilon_a}+1)}=e^{\varepsilon_a},
\end{equation}
and
\begin{equation}
    \frac{\operatorname{Pr}\left[\mathcal{M}_2\left(\mathbf{x}\right) = s\right]}{\operatorname{Pr}\left[\mathcal{M}_2\left(\mathbf{x}^{\prime}\right) = s\right]}=\frac{\operatorname{Pr}\left[x_j=s_j\right]}{\operatorname{Pr}\left[x_j^{\prime}=s_j\right]}\leq \frac{e^{\varepsilon_f}/(e^{\varepsilon_f}+1)}{1/(e^{\varepsilon_f}+1)}=e^{\varepsilon_f}.
\end{equation}
Therefore, we can obtain that information collection satisfies $\varepsilon_a$-edge LDP and $\varepsilon_f$-feature LDP.
\end{proof}

\begin{theorem}
    HoGS satisfies $\varepsilon$-LDP, where $\varepsilon=\varepsilon_a+\varepsilon_f$.
\end{theorem}
\begin{proof}
Recall that HoGS consists of three phases: information collection, topology reconstruction and feature reconstruction. Based on Lemma \ref{lemma} and Theorem \ref{sequential composition}, the phase of information collection achieves $(\varepsilon_a+\varepsilon_f)$-LDP. For topology reconstruction and feature reconstruction, HoGS only post-processes the perturbed data uploaded by the nodes without consuming the privacy budget (according to Theorem \ref{post-processing}). To sum up, HoGS satisfies $\varepsilon$-LDP.
\end{proof}

\subsection{Computational Complexity} \label{Computational Complexity}
We focus on analyzing the computational complexity of the graph synthesis process. The complexity of subsequently training GNNs on the synthetic graph is not discussed, as it employs existing models and standard training protocols. In particular, for HoGS, the computational complexity of information collection phase is $O(n^2+nd)$, since each node needs to perturb every bit in its adjacency list and feature vector. In the topology reconstruction phase, the curator first calculates the cosine similarity of the node feature vectors and then estimates the link probabilities, which have computational complexities of $O(n^2d)$ and $O(n^2)$, respectively.
Feature reconstruction phase aggregates the node features from potential neighbors, with the computational complexity of $O(n^2d)$. Therefore, the total computational complexity of HoGS in the graph synthesis process is $O(n^2d+nd)$.

\subsection{Technical Novelty} \label{Technical Novelty}
In the field of GNNs with local differential privacy, the simultaneous protection for node feature privacy and link privacy has been first studied in the literature \cite{lin2022towards}. However, a major problem is that \cite{lin2022towards} independently calibrates the noisy adjacency matrix by encouraging graph structure sparsity, which is insufficient for training effective GNNs as it only guarantees the overall graph density. Although a recent study \cite{zhu2023blink} improved the accuracy of noisy graph topology, their scheme, as mentioned above, requires further splitting of the privacy budget to node degree information. In our scenario, this requirement significantly affect the performance of GNNs training. In contrast, our framework HoGS only needs to allocate the privacy budget to the adjacency lists and node features, and guarantees the utility through bidirectional denoising. To our knowledge, this is a novel and practical method that has not been explored in existing literature.

\section{Evaluation} \label{Evaluation}
We constructed extensive experiments on real-world datasets to demonstrate the privacy-utility performance of HoGS for GNN training. The experiments are conducted in Python on a machine with two Intel(R) Xeon(R) Gold 6230R CPUs, 314GB RAM, and an NVIDIA A100 80GB GPU running Ubuntu 22.04 LTS. We implement HoGS and other mechanisms in an environment based on the PyTorch and PyTorch Geometric library.

\subsection{Experimental Setup} \label{Experimental Setup}
\noindent
\textbf{Datasets.} We use three publicly available real-world datasets, which are described below:
\begin{itemize}
	\item \textbf{Cora and CiteSeer \cite{yang2016revisiting}.} These are two well-known citation network datasets that are commonly used as benchmarks for node classification. Each node denotes a scientific publication and links stand for the citation relationships. Each publication is associated with a bag-of-words feature vector and a label representing the category.
    \item \textbf{LastFM \cite{rozemberczki2020characteristic}.} This is a social network dataset collected from the music streaming service LastFM. Nodes are users from Asian countries and links denote friendships between users. Each user has a feature vector for favorite artists and a label for home country. Since this dataset is severely imbalanced, we only retain the top 10 classes with the largest sample sizes, as done in \cite{sajadmanesh2021locally}.
\end{itemize}
Table \ref{Datasets} summarizes detailed statistics of all datasets.
\begin{table}[!t]
	\renewcommand{\arraystretch}{1.2}
	\caption{Details of Real-world Datasets}
	\label{Datasets}
	\centering
	\begin{tabular}{c c c c c c}
		\hline
		\textbf{Datasets} & \textbf{Classes} & \textbf{Nodes} & \textbf{Edges} & \textbf{Features}  & $\textbf{Avg. Degree}$\\
		\hline
        Cora & 7 & 2708 & 5278 & 1433 & 3.90\\
		CiteSeer & 6 & 3327 & 4552 & 3703 & 2.74\\
		LastFM & 10 & 7083 & 25814 & 7842 & 7.29\\
		\hline
	\end{tabular}
\end{table}

\noindent
\textbf{Baselines.} We compare the performance of HoGS with the following three baselines:
\begin{itemize}
	\item \textbf{FP-BLink.} The mechanism BLink \cite{zhu2023blink} considers GNN training with edge LDP. Since the original design of Blink only involves the protection of graph topology, we enhance it with our feature perturbation and reconstruction mechanism, from which we get the baseline FP-BLink.
    \item \textbf{FP-LDPGen.} This method LDPGen \cite{qin2017generating} generates a synthetic graph topology to perform graph statistics tasks, which implements edge LDP. Similarly, for comparability, we adapt LDPGen to our scenario by combining it with our mechanisms on node features. We denote this baseline as FP-LDPGen. 
    \item \textbf{Solitude.} Liu et al. proposed Solitude \cite{lin2022towards} as an LDP mechanism to simultaneously protect the links and features of the training graph. However, their feature privacy definition is user-level, which differs from ours. To ensure the rationality of comparison, we reproduce this mechanism under our feature LDP setting.
\end{itemize}
In addition, to further demonstrate the performance of HoGS in non-private feature scenarios, the standard BLink \cite{zhu2023blink} is also included as an additional baseline.\\

\noindent
\textbf{Experimental Settings.} Following previous research \cite{sajadmanesh2021locally,zhu2023blink}, for each dataset, we randomly split all nodes into training/validation/test sets with the ratio of 50$\%$, 25$\%$ and 25$\%$, respectively. We adopt three well-established GNN architectures, including GCN \cite{kipf2017semisupervised}, GraphSAGE \cite{hamilton2017inductive} and GAT \cite{velivckovic2017graph}, as the backbone models for HoGS and other baselines. 
All GNN architectures have the same model structure: two convolutional layers with 16 hidden units and a ReLU operator followed by a dropout layer. The GAT architecture additionally has four attention heads. Note that we did not run the baseline Solitude on the GAT architecture, as it is unreasonable to make all nodes attend to each other \cite{velivckovic2017graph}. To better compare various LDP mechanisms, we also implemented non-private GNNs (which can be viewed as the case of the privacy budget $\varepsilon=\infty$) on all datasets as a theoretical bound of the experimental performance. 

\noindent
\textbf{Parameter Settings.} In our experiments, we vary the total privacy budget $\varepsilon$ from $3$ to $8$, a range widely used in real-world industrial practice under LDP settings \cite{DPTeamApple}. 
For each $\varepsilon$ value, we conducted a grid search on the validation set to find the best-performing hyperparameters for each combination of LDP mechanisms, GNN architectures and datasets. To reduce the randomness introduced by the LDP mechanisms, we performed 5 independent runs for each hyperparameter configuration and selected the configuration with the highest mean accuracy. The hyperparameter search space is set as follows: For non-private GNNs and all mechanisms other than Solitude, the dropout rate is searched from $\{10^{-1},10^{-2},10^{-3},0\}$, learning rate is searched from $\{10^{-1},10^{-2},10^{-3}\}$ and weight decay is searched from $\{10^{-3},10^{-4},10^{-5},0\}$. For HoGS, the parameter $\delta$ is searched from $\{0.1, 0.3, 0.5, 0.7, 0.9\}$, threshold $\tau$ is searched from $\{0.5, 0.7, 0.9\}$ and step parameter $l$ is searched from $\{0,1,2\}$.
For BLink, the privacy parameter $\delta$ is searched from $\{0.1, 0.3, 0.5, 0.7, 0.9\}$. For FP-BLink, given the privacy parameters $\delta_1$ and $\delta_2$, we set the privacy budgets for the  degree, adjacency list and feature vector to $\varepsilon_{d}=\delta_1 (1-\delta_2)\varepsilon$, and $\varepsilon_{a}=\delta_1 \delta_2\varepsilon$ and $\varepsilon_{f}=(1-\delta_1)\varepsilon$, respectively. Then $\delta_1$ and $\delta_2$ are searched from $\{0.25,0.33,0.5,0.66,0.75\}$. To compare under the same $\varepsilon$, for LDPGen and Solitude, we set the parameter $\delta$ similarly to HoGS, such that the privacy budgets for the graph topology and feature vector are $\varepsilon_a=(1-\delta) \varepsilon$ and $\varepsilon_f=\delta \varepsilon$, respectively, and $\delta$ is searched from $\{0.1, 0.3, 0.5, 0.7, 0.9\}$. In addition, for Solitude, the dropout rate and learning rate is searched from $\{10^{-1},10^{-2},10^{-3}\}$, weight decay is searched from $\{10^{-3},10^{-4},10^{-5},0\}$, both $l_x$ and $l_y$ are searched from $\{0,2\}$.

\noindent
\textbf{Metrics.} We evaluate the performance of HoGS and all baselines using model accuracy on the test set. Specifically, we run each experiment 10 times based on the optimal parameters found in the grid search, and report the mean and standard deviation of classification accuracy.

\subsection{End-to-End Evaluation} \label{End-to-End Evaluation}
For each GNN architecture, we first perform an end-to-end evaluation of HoGS and the baselines on all datasets in the scenario of link and node feature protection. 

\noindent
\textbf{Utility on GCN.} Fig. \ref{fig_gcn} shows the average test accuracy of HoGS and all baselines on the GCN model. We can see that the test accuracy of all methods shows an upward trend with the increase of privacy budget, reflecting the privacy-utility trade-off in differential privacy. Moreover, HoGS significantly outperforms all baselines on the three datasets. The reasons are follows: First, HoGS only needs to split the privacy budget into two parts to protect the privacy of link and node features respectively, avoiding further consumption of the privacy budget compared to the baseline FP-BLink. Second, the bidirectional denoising scheme adopted by HoGS effectively reconstructs the graph structure and feature information. For Solitude, we find that it achieves poor performance since it only utilizes the inherent sparsity of the graph to denoise link information. The performance improvement of the baseline FP-LDPGen is minimal because they use a degree vector-based random graph model to synthesize the graph, which is only designed to preserve graph statistics. In addition, when the privacy budget is small (i.e., $\varepsilon=3$), the advantage of HoGS is not significant compared to other methods. This is because the small privacy budget introduces excessive noise, resulting in unsatisfactory bidirectional denoising.

\begin{figure*}[!t]
	\centering
    \includegraphics[width=0.6\linewidth]{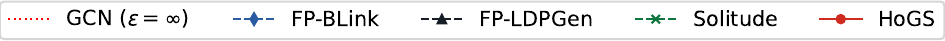}\\
    \vspace{-0.2cm}
	\subfloat[Cora]{
		\includegraphics[width=0.3\linewidth]{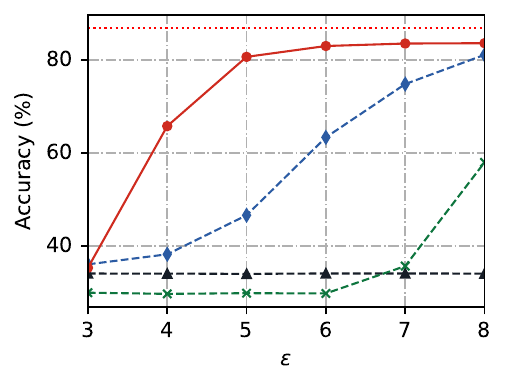}
	}\hspace{-2.6mm}
	\subfloat[CiteSeer]{
		\includegraphics[width=0.3\linewidth]{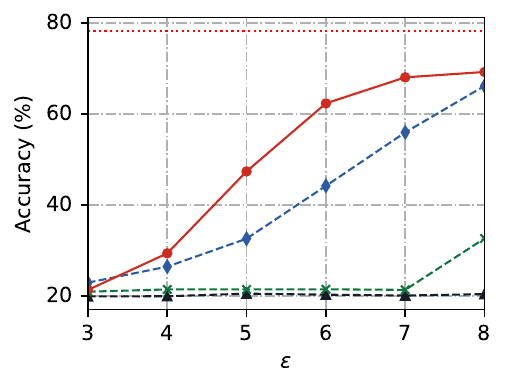}
	}\hspace{-2.6mm}
	\subfloat[LastFM]{
		\includegraphics[width=0.3\linewidth]{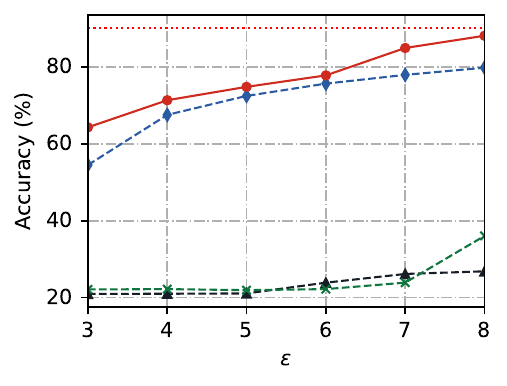}
	}
	\caption{Performance comparison of HoGS with other methods on GCN model given different privacy budgets.}
	\label{fig_gcn}
\end{figure*}

\noindent
\textbf{Utility on GraphSAGE.} Fig. \ref{fig_graphsage} illustrates the average test accuracy of all methods on the GraphSAGE model. We can observe that HoGS achieves higher classification accuracy than the baselines. For the CiteSeer dataset with low average node degree, the LDP perturbation under a small privacy budget severely disrupts the sparse topology, leading to low accuracy. We also found that for the LastFM dataset, FP-LDPGen and Solitude have higher accuracy on the GraphSAGE model than the GCN model. This is because these two baselines preserve the community structure and density of LastFM as a social network, while the GraphSAGE architecture independently transforms node features before aggregation, thus better utilizing the graph structure information preserved by the baseline methods. Furthermore, an interesting phenomenon is that as the privacy budget increases, the performance gap of HoGS on the GraphSAGE model compared to the non-privacy baseline is significantly larger than the gap on the GCN model. The reason is that the parameterized aggregator in GraphSAGE is easily misled by noisy features during training, which leads to the learning of a biased aggregation function and limits the upper bound of performance.

\begin{figure*}[!t]
	\centering
    \includegraphics[width=0.6\linewidth]{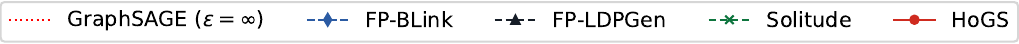}\\
    \vspace{-0.2cm}
	\subfloat[Cora]{
		\includegraphics[width=0.3\linewidth]{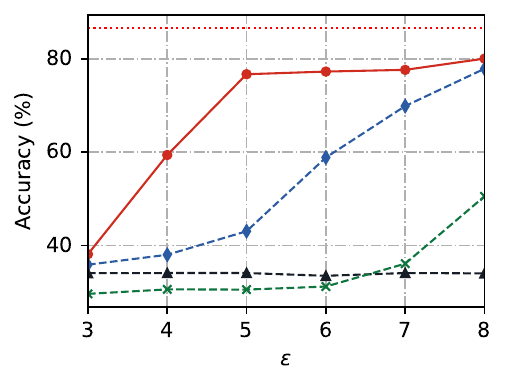}
	}\hspace{-2.6mm}
	\subfloat[CiteSeer]{
		\includegraphics[width=0.3\linewidth]{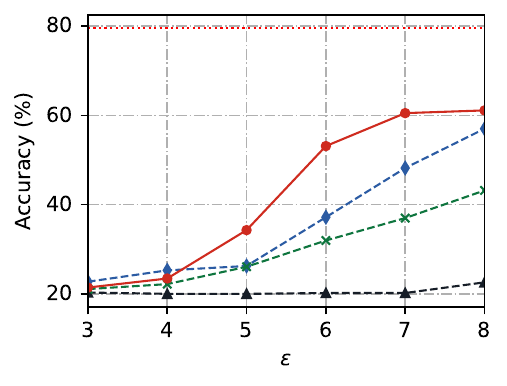}
	}\hspace{-2.6mm}
	\subfloat[LastFM]{
		\includegraphics[width=0.3\linewidth]{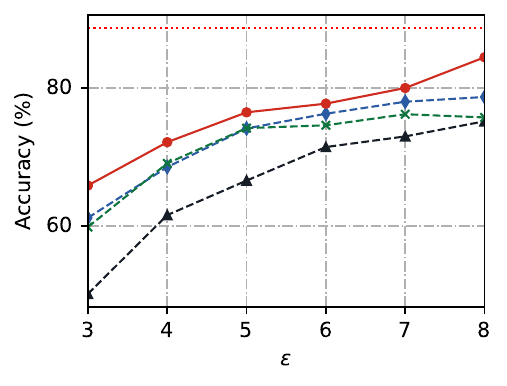}
	}
	\caption{Performance comparison of HoGS with other methods on GraphSAGE model given different privacy budgets.}
	\label{fig_graphsage}
\end{figure*}

\noindent
\textbf{Utility on GAT.} Fig. \ref{fig_gat} shows the average test accuracy of HoGS and the baselines on the GAT model. HoGS obtains state-of-the-art performance across all datasets, which indicates that HoGS better recovers the link and feature information of the original graph. Compared to the other two models, the baseline FP-LDPGen showed no significant performance improvement on all three datasets. This is because the random topology generated by FP-LDPGen cannot provide meaningful learning signals for the attention mechanism of GAT, thus preventing the GAT model from leveraging its advantages. This highlights the key value of our proposed HoGS framework: through Homophily-based joint reconstruction, the synthetic graph generated by HoGS preserves the semantic association between topology and features, which not only improves accuracy but also expands the range of GNN models available under privacy protection.

\begin{figure*}[!t]
	\centering
    \includegraphics[width=0.45\linewidth]{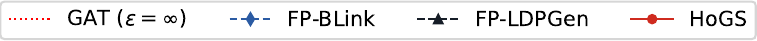}\\
    \vspace{-0.2cm}
	\subfloat[Cora]{
		\includegraphics[width=0.3\linewidth]{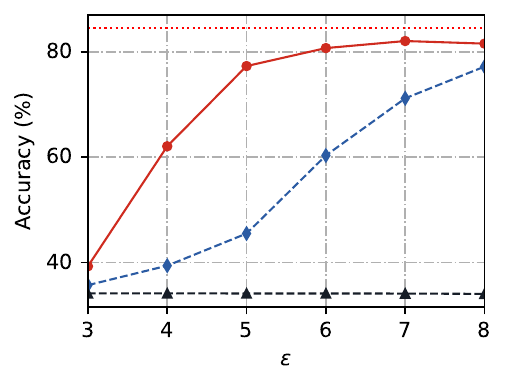}
	}\hspace{-2.6mm}
	\subfloat[CiteSeer]{
		\includegraphics[width=0.3\linewidth]{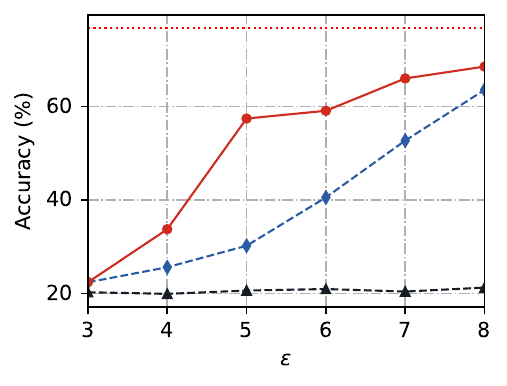}
	}\hspace{-2.6mm}
	\subfloat[LastFM]{
		\includegraphics[width=0.3\linewidth]{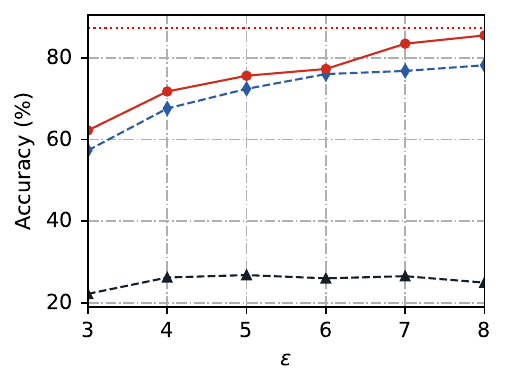}
	}
	\caption{Performance comparison of HoGS with other methods on GAT model given different privacy budgets.}
	\label{fig_gat}
\end{figure*}

\subsection{Evaluation under Non-Private Features} \label{Evaluation under Non-Private Features}
To comprehensively evaluate the performance of HoGS, we further construct experiments in a non-private feature scenario. In this scenario, HoGS adopts the entire privacy budget to perturb adjacency lists while treating node features as public information, thereby implementing edge LDP. We compared it with BLink \cite{zhu2023blink}, a state-of-the-art mechanism in the same privacy scenario. Table \ref{Non-Private Features} shows the experimental results for all datasets and GNN architectures at $\varepsilon=\{3,4,5\}$. We focus on this stricter privacy budget range (lower $\varepsilon$) to align with the financial risk management scenario highlighted in Section \ref{INTRODUCTION}: since graph edges represent highly confidential financial relationships, they demand rigorous privacy protection, making performance under low privacy budgets particularly critical. 

We can observe that HoGS outperforms BLink in most cases. This is because HoGS only needs to inject noise into the adjacency list and denoise it using node features that do not consume privacy budget. In contrast, BLink requires additional privacy budget for node degrees. It is worth noting that under strong privacy settings (e.g., $\varepsilon=3$), HoGS performs worse than BLink on the GAT model, which is mainly due to the unique attention mechanism of GAT. Unlike GCN and GraphSAGE, the neighbor importance weights learned by GAT are highly dependent on node features. With an overly small privacy budget, the highly noisy adjacency list introduces many false neighbors whose features are then used in attention calculations, thereby misleading the model and amplifying errors.

\begin{table*}[!t]
	\renewcommand{\arraystretch}{1.5}
	\caption{Utility Comparison of HoGS and BLink under Non-Private Features}
		\label{Non-Private Features}
	\centering
	\begin{tabular}{c |c c c |c c c |c c c}
			\hline
			\multirow{2}{*}{\makecell[c]{\textbf{Accuracy} \\ \textbf{($\%$)}}} & \multicolumn{3}{c|}{\textbf{Cora}} & \multicolumn{3}{c|}{\textbf{CiteSeer}} & \multicolumn{3}{c}{\textbf{LastFM}} \\
			\cline{2-10}
			& $\varepsilon=3$ & $\varepsilon=4$ & $\varepsilon=5$ & $\varepsilon=3$ & $\varepsilon=4$ & $\varepsilon=5$ & $\varepsilon=3$ & $\varepsilon=4$ & $\varepsilon=5$ \\
			\hline
            \hline
            \multicolumn{10}{c}{GCN}\\
            \hline
			  BLink & $71.0\pm0.5$ & $77.0\pm0.8$ & $84.5\pm1.0$ 
              & $\mathbf{73.8}\pm0.4$ & $73.8\pm0.4$ & $76.7\pm0.8$ 
              & $70.9\pm0.6$ & $70.6\pm1.3$ & $78.9\pm0.8$  \\
			HoGS & $\mathbf{73.3}\pm1.4$ & $\mathbf{82.6}\pm0.8$ & $\mathbf{84.7}\pm0.4$    
            & $66.3\pm1.2$ & $\mathbf{75.6}\pm1.1$ & $\mathbf{78.9}\pm0.5$    
            & $\mathbf{71.0}\pm1.8$ & $\mathbf{81.9}\pm0.6$ & $\mathbf{87.0}\pm0.3$  \\   
            \hline
            \hline
            \multicolumn{10}{c}{GraphSAGE}\\
			\hline
            BLink & $71.7\pm0.8$ & $77.2\pm1.3$ & $84.2\pm0.8$ 
            & $73.8\pm0.5$ & $74.0\pm0.4$ & $77.3\pm0.9$ 
            & $71.9\pm1.1$ & $75.0\pm1.3$ & $79.2\pm1.6$  \\
			HoGS & $\mathbf{77.4}\pm0.9$ & $\mathbf{83.1}\pm0.7$ & $\mathbf{84.9}\pm0.6$    
            & $\mathbf{74.8}\pm1.0$ & $\mathbf{77.1}\pm0.9$ & $\mathbf{79.0}\pm0.6$    
            & $\mathbf{80.5}\pm0.9$ & $\mathbf{84.0}\pm0.7$ & $\mathbf{86.6}\pm0.5$  \\    
            \hline
            \hline
            \multicolumn{10}{c}{GAT}\\
			\hline
            BLink & $\mathbf{71.1}\pm0.5$ & $71.2\pm0.6$ & $81.6\pm1.3$ 
            & $\mathbf{73.5}\pm0.3$ & $\mathbf{73.4}\pm0.3$ & $73.4\pm0.4$ 
            & $\mathbf{71.1}\pm0.7$ & $70.8\pm0.9$ & $78.5\pm0.8$  \\
			HoGS & $63.1\pm3.3$ & $\mathbf{79.8}\pm1.1$ & $\mathbf{82.9}\pm1.0$    
            & $54.6\pm2.6$ & $70.8\pm1.8$ & $\mathbf{76.5}\pm1.1$    
            & $50.9\pm7.9$ & $\mathbf{72.4}\pm3.4$ & $\mathbf{83.9}\pm1.1$  \\    
            \hline
	\end{tabular}
\end{table*}

\subsection{Ablation Study} \label{Feature Reconstruction Analysis}

We perform an ablation study to validate the effectiveness of our reconstruction mechanisms in GNN training. Specifically, we created two variants: (1) BLink + FR, where the topology reconstruction is replaced by the state-of-the-art topology denoising method BLink; and (2) TR + KProp, where the feature reconstruction is replaced by KProp \cite{lin2022towards} -- a advanced scheme that optimizes node features by averaging the feature vectors of $K$-hop neighbors. Fig. \ref{fig_khop} shows the performance of HoGS and different variants when the total privacy budget $\varepsilon=4$. We can find that our method significantly outperforms BLink + FR and TR + KProp for $K=1$ and $K=2$ on all datasets. This demonstrates that our feature similarity-aware topology reconstruction and weighted aggregation-based feature reconstruction mechanisms can effectively mitigate noise in both links and node features, thereby improving the test accuracy of GNNs.

\begin{figure}[!t]
	\centering
    \includegraphics[width=0.6\linewidth]{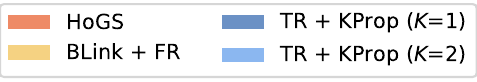}
	\includegraphics[width=0.75\linewidth]{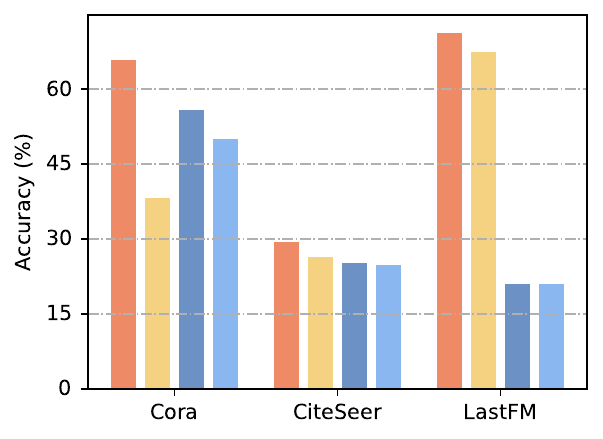}
	\caption{Performance of HoGS with different topology and feature denoising methods.}
	\label{fig_khop}
\end{figure}

\subsection{Parameter Variation} \label{Parameter Variation}
\noindent
\textbf{Effect of Privacy Budget Allocation.} We analyze the effect of privacy budget allocation on GNN performance by varying the privacy parameter $\delta$. Fig. \ref{fig_delta} shows the classification accuracy of graph convolutional network on the Cora and Lastfm datasets as $\delta$ ranges from 0.1 to 0.9. We found that the impact of $\delta$ on accuracy varies significantly across different datasets. Specifically, for the Cora dataset, a smaller $\delta$ yields better utility, whereas for the LastFM dataset, a larger $\delta$ improves model performance by preserving feature quality. This discrepancy arises from the varying modality dependency of different datasets in the node classification task. As a typical citation network, the classification performance on Cora heavily relies on the structural integrity of the graph topology. Therefore, a smaller $\delta$ (implying a higher topology privacy budget $\varepsilon_a$) is more advantageous. Conversely, the LastFM dataset is characterized by extremely high feature dimensionality, and its node features (users' favorite artists) contain highly discriminative information for inferring the users' home countries. Therefore, a larger $\delta$ (implying a higher feature privacy budget $\varepsilon_f$) is required to mitigate the degradation of the signal-to-noise ratio caused by high-dimensional feature perturbation.

\begin{figure}[!t]
	\centering
    \includegraphics[width=0.56\linewidth]{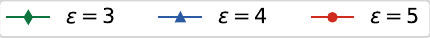}
	\subfloat[Cora]{
		\includegraphics[width=0.46\linewidth]{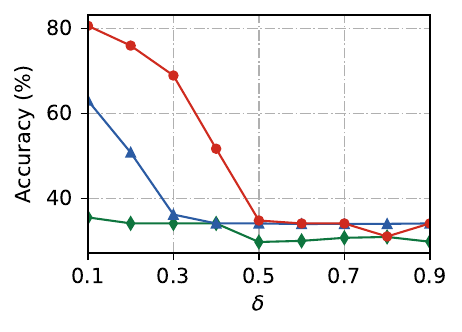}
	}\hspace{-4mm}
	\subfloat[LastFM]{
		\includegraphics[width=0.46\linewidth]{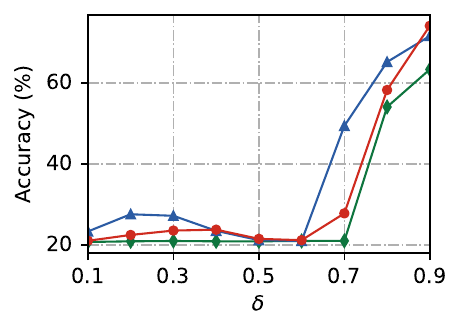}
	}
	\caption{Effect of privacy parameter $\delta$ on GCN test accuracy.}
	\label{fig_delta}
\end{figure}

\noindent
\textbf{Effect of Threshold $\tau$.} Recalling Section \ref{Topology Reconstruction}, we determine the links in the synthetic graph topology based on the threshold $\tau$. To better understand the effect of different $\tau$ on GNN performance, we report the accuracy of graph convolutional networks on the Cora and Lastfm datasets as $\tau$ varies from 0.5 to 0.9, as shown in Fig. \ref{fig_tau}. For the Cora dataset, we observe that the test accuracy of GNNs exhibits a relatively stable or slightly increasing trend as $\tau$ increases across different privacy budgets. This is primarily attributed to the significant sparsity of Cora: under LDP mechanisms, sparse graphs are more susceptible to random noise, resulting in numerous spurious connections. Increasing $\tau$ can filter out some noisy edges, thereby improving the utility of the trained GNN. For denser networks like LastFM, the impact of $\tau$ is closely related to the privacy budget $\varepsilon$. Under relatively high privacy budgets, appropriately increasing $\tau$ (e.g., to 0.7) can significantly filter noise while preserving strong connections, leading to a substantial performance boost. However, under low privacy budgets, since the connection probability of real links are generally suppressed by noise, a high $\tau$ causes the loss of some structural information, resulting in a slight decrease in model utility.

\begin{figure}[!t]
	\centering
    \includegraphics[width=0.56\linewidth]{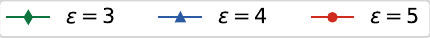}
	\subfloat[Cora]{
		\includegraphics[width=0.46\linewidth]{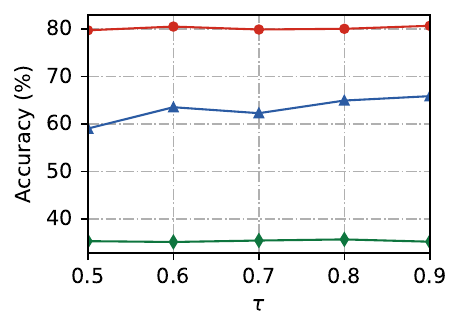}
	}\hspace{-2.6mm}
	\subfloat[LastFM]{
		\includegraphics[width=0.46\linewidth]{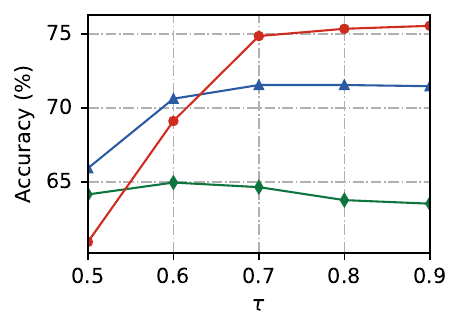}
	}
	\caption{Effect of threshold $\tau$ on GCN test accuracy.}
	\label{fig_tau}
\end{figure}

\section{Related Work} \label{Related Work}
\subsection{Graph Neural Networks} \label{Graph Neural Networks}

Graph neural networks have been widely used in recent years for various graph-based learning tasks, including node classification and link prediction. A series of influential GNN models have been proposed, such as Graph Convolutional Networks \cite{kipf2017semisupervised}, GraphSAGE \cite{hamilton2017inductive}, Graph Attention Networks \cite{velivckovic2017graph}, and Jumping Knowledge Networks \cite{xu2018representation}. Since our graph generation scheme is independent of the GNN learning process, we refer the audience to existing surveys for details on GNN models, including their performance and applications.

\subsection{GNNs via Differential Privacy} \label{End-to-End Evaluation}
There have been some studies on differentially private GNNs. Specifically, under the central DP model, Wu et al. \cite{wu2022linkteller} studied the link-stealing attacks on GCNs and proposed a DP mechanism for GCN training, called DpGCN, which provides DP guarantees for edges in the graph. Kolluri et al. \cite{kolluri2022lpgnet} further improved the utility of GCN-based node classification under edge DP by inserting graph structure information into the multi-layer perceptrons (MLPs) model that only uses node features. Sajadmanesh et al. \cite{sajadmanesh2023gap} proposed a novel GNN architecture that can provide both edge DP and node DP, where they add noise to the low-dimensional aggregation embeddings one-time and cache them to save privacy budget. Recently, Qi et al. \cite{qi2025hiding} explored the problem of protecting partially sensitive node features in a setting where graph structures are publicly available. Under the local DP (LDP) model, Sajadmanesh and Gatica-Perez \cite{sajadmanesh2021locally} first presented the Locally Private GNN to protect the privacy of node features and labels but not the graph structure. Hidano and Murakami \cite{hidano2022degree} proposed a degree-preserving LDP mechanism to protect edges in GNNs. Zhu et al. \cite{zhu2023blink} proposed BLink, which achieves edge LDP by independently injecting noise into the adjacency list and node degree. Pei et al. \cite{pei2023privacy} designed an LDP mechanism based on local graph enhancement to protect a subset of node features. The aforementioned studies on GNNs with LDP are limited to protecting either node features or graph structure, but not both. Lin et al. \cite{lin2022towards} investigated dual protection for node features and edges, which is based on graph sparsity to denoise the topology and adopts the scheme in \cite{sajadmanesh2021locally} to protect node features. However, their LDP mechanism has different privacy definition for node features than ours and results in poor privacy-utility tradeoffs.

\subsection{Graph Analysis via Differential Privacy} \label{End-to-End Evaluation}
In the field of differential privacy graph analysis, many works aim to estimate specific graph statistics such as subgraph counts \cite{sun2019analyzing,imola2021locally,imola2022differentially,liu2024edge}, relationship frequency \cite{wang2023accurately} and clustering coefficient \cite{wang2013learning,ye2020lf}. There are also several works dedicated to publishing a synthetic graph, which are closely related to our research. Specifically, Qin et al. \cite{qin2017generating} presented a multi-stage technique based on degree vectors to gradually cluster nodes and construct a synthetic graph, while satisfying edge LDP. Jian et al. \cite{jian2021publishing} proposed two novel algorithms, including node perturbation and a more utility-preserving edge perturbation, to achieve graph publication with node DP. Yuan et al. \cite{yuan2023privgraph} achieved graph generation with edge DP by dividing the graph into multiple communities and then extracting intra-community and inter-community information. Brito et al. \cite{brito2023global} developed two mechanisms for publishing count-weighted graphs under central and local DP, which combine geometric noise addition, statistical preservation, and post-processing to enhance data utility. Xu et al. \cite{xu2024differentially} studied the problem of dynamically weighted graph publishing with DP, which aims to generate a sequence of weighted graph snapshots. However, all of these efforts on synthetic graph publication are focused on downstream statistical analysis tasks, which are not useful for training GNNs.

\section{Conclusion} \label{Conclusion}
In this paper, we propose HoGS, a novel local differentially private graph synthesis framework that preserves the privacy of link and node features while maintaining the accuracy of graph neural network training. The framework consists of three key components: information collection, topology reconstruction, and feature reconstruction, which leverage the homophily of graphs to achieve a better balance between privacy preservation and utility loss. Extensive experiments on multiple datasets and GNN architectures demonstrate that HoGS significantly outperforms other methods. An important future direction is to extend privacy preservation to richer information types, such as edge weights, to achieve more comprehensive privacy protection for graph data.

\bibliographystyle{IEEEtran}
\bibliography{IEEEabrv,mybib2}

\vspace{11pt}

\vfill

\end{document}